\documentclass[us]{article} 
\usepackage[numbers, compress]{natbib}

\usepackage[accepted]{aistats2022}
\usepackage{url}
\usepackage{amsthm}
\usepackage{amssymb,dsfont,amsmath,amsthm,mathtools,natbib}
\usepackage{enumitem}
\usepackage{bbm}

\usepackage{algorithm}
\usepackage{algorithmic}
\usepackage{dcolumn}
\usepackage{tikz}
\usetikzlibrary{positioning}
\usepackage{booktabs}

\newtheorem{theorem}{Theorem}
\newtheorem{definition}{Definition}
\newtheorem{proposition}{Proposition}
\newtheorem{lemma}{Lemma}

\newtheorem{remark}{Remark}

\usepackage{xspace}
\newcolumntype{d}[1]{D{.}{.}{#1}}

\newcommand{\norm}[1]{|\!| #1 |\!|}
\def\R{\mathbb{R}}
\def\1{\mathbbm{1}}

\def\X{\mathbf{X}}
\def\u{\mathbf{u}}
\def\v{\mathbf{v}}
\def\y{\mathbf{y}}

\def\E{\mathbf{E}}

\def\s{\mathbf{s}}

\def\V{\mathbf{V}}

\newcommand{\Ex}{{\rm I\kern-.3em E}}

\def\cN{\mathcal{N}}

\def\cP{\mathcal{P}}

\usepackage[colorinlistoftodos,bordercolor=orange,backgroundcolor=orange!20,linecolor=orange,textsize=scriptsize]{todonotes}

\DeclareMathOperator{\complex}{complexity}

\begin{document}

\twocolumn[

\aistatstitle{Statistical and Topological Properties  of Gaussian 
Smoothed Sliced Probability Divergences}

\aistatsauthor{
	Alain Rakotomamonjy$\dagger$\\
	Criteo AI Lab, Paris\\
	\texttt{alain.rakoto@insa-rouen.fr} \\
	\And
	Mokhtar Zahdi  El Alaya$\dagger$	 \\
	LMAC, Université Technologique de Compiègne  \\
	\texttt{mokhtarzahdi.alaya@gmail.com} \\
	\AND  
	Maxime Berar \\
	LITIS, Université de Rouen \\
	\texttt{maxime.beraro@univ-rouen.fr} \\
	\And
	Gilles Gasso \\
	LITIS, INSA de Rouen \\
	\texttt{gilles.gasso@insa-rouen.fr} \\
}
~\\~\\
 ]

\begin{abstract}
	Gaussian smoothed sliced Wasserstein distance has been recently introduced for comparing
	probability distributions, while preserving privacy on the data.
	It has been shown, in applications
	such as domain adaptation,  to  provide  performances similar to its non-private
	(non-smoothed) counter-part. However, the computational and statistical properties of such a
	metric is  not yet been well-established. In this paper, we analyze the theoretical properties
	of this distance as well as those of generalized versions denoted as Gaussian smoothed sliced divergences.
	We show that smoothing and slicing preserve  the metric property and the weak topology. We also
	provide results on the sample complexity of such divergences.  Since, the privacy level depends on the
	 amount of Gaussian smoothing, we analyze  the impact
	 of this parameter on the divergence.
	 We support our theoretical findings with
	 empirical studies of Gaussian smoothed and sliced version of Wassertein distance, Sinkhorn divergence
	 and maximum mean discrepancy (MMD). In the context of privacy-preserving domain adaptation, we confirm that
	  those Gaussian smoothed sliced Wasserstein and MMD divergences  perform very well while ensuring
	data privacy.
\end{abstract}

\section{Introduction}
\label{sec:back}

Divergences for comparing two distributions have been shown to be important
for achieving good performances in the contexts  of generative modeling~\cite{pmlr-v70-arjovsky17a,salimans2018improving}, domain adaptation~\cite{long2015learning,courty2016optimal,lee2019sliced}, and in computer vision~\citep{bonnel2011,solomon2015} among many more applications~\citep{klouri17,peyre2019COTnowpublisher}.
Examples of divergences that have
been proven to be useful for those tasks are maximum mean discrepancy~\cite{gretton2012kernel, long2015learning, sutherland2017generative}, Wasserstein
distance~\citep{monge1781,kantorovich1942,villani09optimal} or its variant the sliced Wasserstein distance (SWD)~\cite{bonneel2019spot,kolouri2019gensliced,nguyen2020distributional,kolouri2016slicedkernels}.

Sliced Wasserstein distance has the advantage of being computationally
efficient as it exploits a closed-form solution for distributions with support
on $\R$, by computing the expectation of one-dimensional random projections
of distributions in $\R^d$. Owing to this efficiency and the resulting
scalability, this distance has been successfully
applied in several applications ranging from generative models to domain adaptation~\cite{kolouri2018sliced,max-SW,wu2019sliced,lee2019sliced} 
and its statistical property has been well-studied \cite{nadjahi2020}.

A differentially private variant of sliced Wasserstein distance has been recently introduced in~\cite{pmlr-v139-rakotomamonjy21a}, 
for comparing distributions in relation to sensitive applications in which training data can not
be disclosed.  Privacy through a so-called Gaussian mechanism is induced by adding Gaussian noise to each 1D projection of each distribution, leading to the so-called Gaussian-smoothed sliced Wasserstein distance.  This relationship between Gaussian smoothing and privacy has also been 
mentioned by \citet{nietert2021smooth} as future work to address,
while they analyzed 
the structural and statistical behavior of Gaussian smoothed Wasserstein distances.

However, up to now, theoretical
properties of this Gaussian smoothed sliced Wasserstein distance   are not well fully understood except for its metric properties \cite{pmlr-v139-rakotomamonjy21a}.
In this work, we  investigate those theoretical properties and the one of
more general Gaussian smoothed sliced divergences. Indeed, given 
a base distance or divergence for distributions in $\R^d$, we can introduce
its related Gaussian smoothed sliced divergence.
Specifically, the theoretical properties of interest are the metric property and
the underlying topology. From a statistical point of view, we 
seek at understanding the relationship between the sample complexity of 
the base divergence and its Gaussian smoothed sliced version.
Regarding privacy, the role of the Gaussian  smoothing is 
of primary importance as it induces the privacy level achieved by the divergence. Hence, we also provide an analysis on its impact with respect to the standard deviation of the Gaussian noise.
For supporting our theoretical study, we provide some numerical experiments on 
toy problem, and we also provide some numerical study on 
domain adaptation illustrating how owing to the
topology induced by our metric, differential privacy comes
almost for free (without loss of performances) in this context.

The paper is organized as follows, after introducing the 
notations and some background in Section \ref{sec:back}, we detail
the topological properties of Gaussian smoothed sliced divergence in 
Section \ref{subsec:topo} while the statistical properties are established in Section
\ref{subsec:stats}. Experimental analyses for supporting the theory and 
showcasing the relevance of our divergences in a domain adaptation
situation are depicted in Section \ref{sec:expe}. Discussions on the perspectives
and limitations  are in Section \ref{sec:conclu}.

\section{Preliminaries}
\label{sec:back}

For the reader's convenience, we provide a brief summary of the standard notations and the definitions that will be used throughout the paper.

\paragraph{Notations.} For $d \in \mathbb{N}^*$,   let $\cP(\R^d)$ be the set of Borel probability measures on  $\R^d$ and $\cP_p(\R^d) \subset \cP(\R^d)$, those with finite moment of
order $p$, {i.e.,} $\cP_p(\R^d) \triangleq \{\mu \in \cP : \int \|x\|^p d\mu(x) < +\infty\}$, where $\|\cdot\|$ is the Euclidean norm and $\langle \cdot, \cdot \rangle$ is the Euclidean inner-product.
For two probability distributions $\mu$ and $\nu$, we denote their convolution as
$\mu * \nu \in \cP(\R^d)$ and by definition, we have
$(\mu * \nu)(A) = \int_x\int_y \mathbf{1}_A(x+y) d\mu(x) d\nu(y)$, where 
$\mathbf{1}_A(\cdot)$ is the indicator function over $A$. Given two independent random
variables $X \sim \mu$ and $Y \sim \nu$, we remind that $X + Y \sim \mu * \nu$.

The $d$-dimensional unit-sphere is noted as $\mathbb{S}^{d-1} \triangleq \{\theta \in \R^d : \|\theta\| = 1 \} $. We denote by $u_d$ the uniform distribution on $\mathbb{S}^{d-1}$ and we use $\delta(\cdot)$ to denote the Kronecker delta function. 
We note as $\E_\mu f$ the expectation of the function $f$ with respect to $\mu$. 
Hence, the characteristic function of a probability distribution
$\mu \in \cP(\R^d)$  is $\varphi_\mu(t) = \E_\mu[e^{iX^\top t}]$. Given this definition,
similarly to the Fourier transform, the characteristic function of the convolution of
two probability distributions has the following form 
$\varphi_{\nu * \mu}(t) =\varphi_{\nu}(t)\cdot\varphi_{\mu}(t)$.

\paragraph{Sliced Wasserstein Distance.} We remind in this paragraph several measures of similarity between two distributions.
The Wasserstein distance of order $p \in [1, \infty)$ between two measures in $\mathcal{P}_p(\R^d)$ is given by the
relaxation of the optimal transport problem,
and it is defined as
\begin{equation*}\label{eq:wd}
W_p^p(\mu,\nu) = \inf_{\gamma \in \Pi(\mu,\nu)} \int_{\R^d \times \R^d} 
\|x - x^\prime\|^p \gamma(x,x^\prime) dxdx^\prime
\end{equation*} 
where $\Pi(\mu,\nu)\triangleq \{ \gamma \in \mathcal{P}(\R^d \times \R^d) |  \pi_{1\#} \gamma=\mu,\pi_{2\#} \gamma=\nu\}$ and $\pi_1, \pi_2$ are
the marginal projectors of $\gamma$ on each of its coordinates. 
When $d=1$, the Wasserstein distance can be computed in a closed-form  owing to the
cumulative distributions of $\mu$ and $\nu$~\cite{rachev1998mass}.
Note that the superscript in $W_p^p$ refers to the power $p$.
 In practice for empirical distributions, the closed-form solution needs just the sorting of samples, which makes it very efficient. Due to this efficiency, efforts have been devoted to derive a metric for high-dimensional distributions based on 1D Wasserstein distance.
The main idea is to project high-dimensional probability distributions onto a random 1-dimensional space and
then to compute the Wasserstein distance. That  operation
can be theoretically formalized through the use of the Radon transform, leading to the so-called sliced Wasserstein distance~\cite{bonneel2019spot,kolouri2019gensliced,nguyen2020distributional,kolouri2016slicedkernels}.
\begin{definition}  For any $p \in [1, \infty)$ and two measures
	$\mu$, $\nu\in \mathcal{P}_p(\R^d)$, the sliced Wasserstein distance (SWD) reads as 
$$
\text{SWD}_p^p(\mu,\nu) \triangleq \int_{\mathbb{S}^{d-1}}  W_p^p(\mathcal{R}_\u \mu,\mathcal{R}_\u \nu) u_d(\u)d\u.
$$
where $\mathcal{R}_\u$ is the Radon transform of a probability distribution, namely
$\mathcal{R}_\u \mu(\cdot) = \int_{\R^d} \mu(\s) \delta(\cdot - \s^\top \u)d\s
	\label{eq:radon}$.

\end{definition}
In practice, the integral is approximated through a Monte-Carlo simulation leading to a sum of 1D Wasserstein distances over a fixed number of random directions $\u$.

\paragraph{Gaussian Smoothed Sliced Wasserstein Distance.} 

Based on this definition of SWD, replacing the Radon projected measures with their Gaussian-smoothed counterpart leads to the following definition:
\begin{definition}
	The $\sigma$-Gaussian smoothed $p$-Sliced Wasserstein distance between probability distributions $\mu$ and $\nu$ in  $\mathcal{P}_p(\R^d)$ is 
	\begin{equation*}
		\label{eq:gsmoothedwass}
	G_\sigma\text{SWD}_p^p(\mu,\nu) \triangleq \int_{\mathbb{S}^{d-1}}  W_p^p(\mathcal{R}_\u \mu * \mathcal{N}_\sigma,\mathcal{R}_\u \nu* \mathcal{N}_\sigma) u_d(\u)d\u.
	\end{equation*}
\end{definition}
It is important to note here that the smoothing (convolution) operation occurs after projection onto the one-dimensional space. Hence, assuming $X \sim \mu$, $Y \sim \nu$ in the integral, for a given $\u$, we compute the 1D Wasserstein distance  between the probability laws of $\u^\top X + Z$ and $\u^\top Y + Z^\prime$ with $Z,Z^\prime \sim \mathcal{N}_{\sigma}$ being independent random variables. 
The metric properties of $G_\sigma\text{SWD}_p^p$ for $p \geq 1$, of this Gaussian smoothed sliced Wasserstein distance have been discussed in a recent work~\cite{pmlr-v139-rakotomamonjy21a}. This latter work has also shown, in the context of differential privacy, 
the importance of convolving the Radon projected distribution with a
Gaussian instead of computing the sliced Wasserstein distance of the original distribution smoothed with a d-dimensional Gaussian 
$\mu * \mathcal{N}_\sigma$.   

\paragraph{Gaussian Smoothed Sliced Divergence.}  
The idea of slicing high-dimensional distributions before feeding them to a divergence between probability distributions can be extended
to other distance than Wasserstein distance. Those sliced
divergences have been studied by~\cite{nadjahi2020}. In a similar way, 
we can define a Gaussian smoothed  sliced divergence, given a 
divergence $D : \cP(\R^d) \times \cP(\R^d) \rightarrow \R^+$ for $d \geq 1$ as:
\begin{definition} 
	The $\sigma$-Gaussian smoothed $p$-Sliced Divergence between probability distributions $\mu$ and $\nu$ in $\cP_p(\R^d)$
	associated to the divergence  $D \triangleq D_\R$, $p\geq 1$  is 
	\begin{equation*}
		\label{eq:gsmootheddiv}
G_\sigma S\text{D}^p(\mu,\nu) \triangleq \int_{\mathbb{S}^{d-1}}  D^p(\mathcal{R}_\u \mu * \mathcal{N}_\sigma,\mathcal{R}_\u \nu* \mathcal{N}_\sigma) u_d(\u)d\u. 
\end{equation*}    
where the superscript $p$ refers to a power.
\end{definition}

Typical relevant divergence is the maximum mean discrepancy (MMD) \cite{gretton2012kernel} or the Sinkhorn divergence \cite{genevay2018learning,peyre2019COTnowpublisher}. In Section \ref{sec:expe}, we report empirical findings based on these divergences as well as on the Wasserstein distance.

\section{Theoretical Properties}
\label{sec:theory}
In this section, we will analyze the properties of 
the Gaussian smoothed sliced divergence, in term of
topological and statistical properties and the influence
of the Gaussian smoothing parameter $\sigma$ on the distance.
\subsection{Topology} \label{subsec:topo}

It has already been shown in~\cite{pmlr-v139-rakotomamonjy21a} that the Gaussian smoothed sliced 
Wasserstein is a metric on $\cP(\R^d)$. In the next, we extend these results to any divergence $D(\cdot,\cdot)$  under some assumptions.

\begin{theorem} 
\label{theorem:proof_topology}

For any $p \in [1, \infty ) $ and $\sigma>0$,  the following properties hold:
	\begin{enumerate}
\item if $D(\cdot,\cdot)$ is non-negative (or symmetric), then $G_\sigma S\text{D}^p(\cdot,\cdot)$ is non-negative (or symmetric);
\item if the base divergence $D(\cdot,\cdot)$ satisfies
the identity of indiscernibles, for $\mu^\prime, \nu^\prime \in \cP(\R) $, $D(\mu^\prime, \nu^\prime) = 0$ if and only if 
$\mu^\prime = \nu^\prime$, then this identity also holds for $G_\sigma S\text{D}^p(\cdot,\cdot)$ for any $\mu,\nu \in \mathcal{P}(\R^d)$;
\item if the  $D(\cdot,\cdot)$ satisfies the triangle inequality then its Gaussian smoothed sliced version $G_\sigma S\text{D}^p(\cdot,\cdot)$ satisfies the triangle inequality.
	\end{enumerate}

\end{theorem}

The above theorem shows that under mild hypotheses over the base divergence $D$, as being a metric for instance, the metric property of its Gaussian smoothed sliced version naturally derives. As exposed in the appendix, the more involved property to prove is the identity of indiscernibles.

Now, we establish under which  conditions on the divergence $D$, the convergence of a sequence in
$G_\sigma S D^p$ implies weak convergence in $\cP(\R^d)$.

\begin{theorem} Let $\sigma \geq 0, p \in [1, \infty)$, $\mu
\in \cP_p(\R^d)$, the sequence of distributions $\{\mu_k \in \cP_p(\R^d) \}_{k \geq 1}$. Assume that the divergence $\text{D}$ metricizes the weak topology. 
Then, $ \lim_{k \rightarrow \infty }G_\sigma S\text{D}^p(\mu_k,\mu) = 0 $ if and only if
 $\{\mu_k\}_k$ converges weakly to $\mu$ \emph{i.e.}, if
 for any $f$ in the set of bounded and continuous functions,  
 $\mu_k \rightarrow \mu$ if $\int_{\R^d} f(x) d \mu_k(x) \rightarrow 
\int_{\R^d} f(x) d \mu(x)$.	 
	
\end{theorem}

\begin{proof}
By using results from~\cite{nadjahi2020}, we know that if $D$ metricizes the weak topology for $\cP(\R)$ then the weak convergence in $\cP(\R)$ is equivalent to the convergence under $D$. Hence,  we have
$$	 G_\sigma SD^p(\mu_k,\mu) \rightarrow 0
	\Leftrightarrow  \mu_k * \cN_\sigma \rightarrow \mu * \cN_\sigma$$
then, using the convolution property of characteristic function gives
	$$\varphi_{\mu_k}(t)\varphi_{\cN_\sigma}(t) 
	\rightarrow  \varphi_{\mu}(t)\varphi_{\mathcal{N}_\sigma}(t) \quad \forall t.$$
	This means that for all $t$, $\varphi_{\mu_k}(t) \rightarrow  \varphi_{\mu}(t) $ 
which concludes the proof, owing to the one-to-one correspondence between characteristic 
functions.
\end{proof}

\subsection{Statistical properties} \label{subsec:stats}

The next theoretical question we are interested in is about the error
we made when the true distribution $\mu$ is approximated by its empirical distribution $\hat \mu$. Such a case is common in practical applications where only (high-dimensional) empirical samples are at disposal. 
 Specifically, we are interested in quantifying two key properties of empirical Gaussian smoothed
 divergence: {\it (i)} 
 the convergence of 
 $G_\sigma S D^p( \hat \mu_n,\hat \nu_n)$ to $G_\sigma S D^p(\mu,\nu)$
{\it (ii)} 
the convergence of 
$ \widehat{G_\sigma S D^p} (\mu,\nu)$ to $G_\sigma S D^p(\mu,\nu)$, {i.e.,} when
approximating the expectation over the random projection with sample mean.

\subsubsection{Sample complexity} 

Herein, our goal is to quantify the error made when approximating $G_\sigma S D^p(\mu,\nu)$ with ${G_\sigma S D^p} (\hat\mu_n,\hat\nu_n)$, where $\hat\mu_n,\hat\nu_n$ are the empirical counterparts of $\mu, \nu$ defined {over $n$ samples}. More precisely, we are interested in establishing an order of the convergence rate of  ${G_\sigma S D^p} (\hat\mu_n,\hat\nu_n)$ towards $G_\sigma S D^p(\mu,\nu)$, according to the number of samples $n.$ This rate stands for the so-called {\it sample complexity.} 

\begin{theorem} 
\label{theorem:sample_complexity}
Fix $p \in [1, \infty)$ and assume that for any $\mu'\in \mathcal{P}(\R)$ with empirical measure $\hat\mu^{'}_n$, $\E[D^p( \hat\mu^{'}_n,\mu')] \leq \alpha_n(p)$. Then, for any $\mu\in \mathcal{P}(\R^d)$ with empirical measure $\hat\mu_n$,
\begin{equation*}
\E[G_\sigma SD^p( \hat \mu_n,\mu)] \leq \alpha_n(p).
\end{equation*}
Additionally, if $D^p$ is a pseudo-metric (non-negative, symmetric with triangle inequality), then 
\begin{equation*}
\E[|G_\sigma SD^p( \hat \mu_n,\hat \nu_n) - G_\sigma SD^p(\mu,\nu)|] \leq 2 \alpha_n(p).
\end{equation*}
\end{theorem}
\begin{proof}
We have 
\begin{align*}
\E&[G_\sigma SD^p( \hat \mu_n,\mu)]\\
& = \E\bigg[\int_{\mathbb{S}^{d-1}}D^p(\mathcal{R}_\u \hat\mu_n * \mathcal{N}_\sigma,\mathcal{R}_\u \mu* \mathcal{N}_\sigma) u_d(\u)\text{d} \u \bigg]\\
&\leq \int_{\mathbb{S}^{d-1}}\E\big[D^p(\mathcal{R}_\u \hat\mu_n * \mathcal{N}_\sigma,\mathcal{R}_\u \mu* \mathcal{N}_\sigma) u_d(\u)\text{d} \u \big]\\
&\leq \int_{\mathbb{S}^{d-1}}\alpha_n(p) u_d(\u)\text{d} \u = \alpha_n(p).
\end{align*}
The triangle inequality entails that, $G_\sigma SD^p( \hat \mu_n,\hat \nu_n) \leq G_\sigma SD^p( \hat \mu_n, \mu) + G_\sigma SD^p( \mu,\nu) + G_\sigma SD^p( \nu,\hat \nu_n),$ which entails, by taking expectation with respect to $\hat \mu_n, \hat \nu_n,$
\begin{align*}
\E\big[|G_\sigma SD^p(& \hat \mu_n,\hat \nu_n) - G_\sigma SD^p(\mu,\nu)|\big]\\
&\leq \E\big[G_\sigma SD^p( \hat \mu_n,\mu)\big] + \E\big[G_\sigma SD^p(\nu,\hat \nu_n)\big]\\
&\leq \E\big[G_\sigma SD^p( \hat \mu_n,\mu)\big] + \E\big[G_\sigma SD^p(\hat \nu_n,\nu)\big]\\
& \leq 2 \alpha_n(p),
\end{align*}
which completes the proof.
\end{proof}
\begin{remark}
Given any base divergence $D^p$, Theorem~\ref{theorem:sample_complexity} shows that the sample complexity of $G_\sigma SD^p$ is proportional to the one dimensional sample complexity of $D^p$. 
\end{remark}

Next, we focus on the sample complexity for the special case of Gaussian smoothed sliced Wasserstein distance. We also provide the convergence rate of $G_\sigma SWD^p( \hat \mu_n,\hat \nu_n)$ towards $G_\sigma SWD^p( \mu, \nu)$. 

\begin{proposition}
\label{corollary:sampComplexGSWD}
For any $p, q \in [1, \infty)$ such that $q > p,$ consider $\mu, \nu \in \mathcal{P}_q(\R^d)$ with its empirical measure $\hat \mu_n.$ Then, the following holds
\begin{align*}
\E&[|G_\sigma SWD^p( \hat \mu_n,\hat \nu_n) - G_\sigma SWD^p(\mu,\nu)|] \leq \alpha_n(p,q, \sigma)
\end{align*}
where 
\begin{align*}
\alpha_n&(p,q,\sigma)\\
&= 2C_{p,q}\times\begin{cases}
2^{p(q-1)/q}(M_q(\mu,\nu) + M_q(\mathcal{N}_\sigma))^{p/q}{\bf{1}}_{q \in 2\mathbb{N}^*},\\
2^{p(q-1)/q}M_q(\mu,\nu))^{p/q}{\bf{1}}_{q \in 2\mathbb{N}+1},\end{cases}
\\
& \qquad \qquad \times 
\begin{cases}
n^{-1/2}{\bf{1}}_{q >2p }, \\
n^{-1/2}\log(n){\bf{1}}_{q=2p}\\
n^{-(q-p)/q}{\bf{1}}_{q \in (p ,2p)}\end{cases}
\end{align*}
and where $M_q(\mu,\nu) = M_q(\mu) + M_q(\nu),$
$C_{p,q}$ is a positive constant  depending only $p,q$, and $M_{q}(\mathcal{N}_\sigma)$ stands for the  $q$-th moment of $\mathcal{N}_\sigma$, that is 
\begin{equation*}
M_{q}(\mathcal{N}_\sigma)){\bf{1}}_{q \in 2\mathbb{N}^*} = \frac{(2q)!}{2^qq!}\sigma^{2q}= 1\cdot2\cdot3\cdots(2q-1)\sigma^{2q}.
\end{equation*}
\end{proposition}
The latter theoretical results show that empirical Gaussian smoothed Wasserstein distance converges at a rate of order $n^{-1/2}$ in the best scenario. 
It is worth also noting that the sample complexity depends
on the amount of smoothing through the moment of the Gaussian
noise : the larger the amount of smoothing, the worse is the 
constant of the complexity.

\begin{figure*}[t]
	\centering
	\includegraphics[width=7cm]{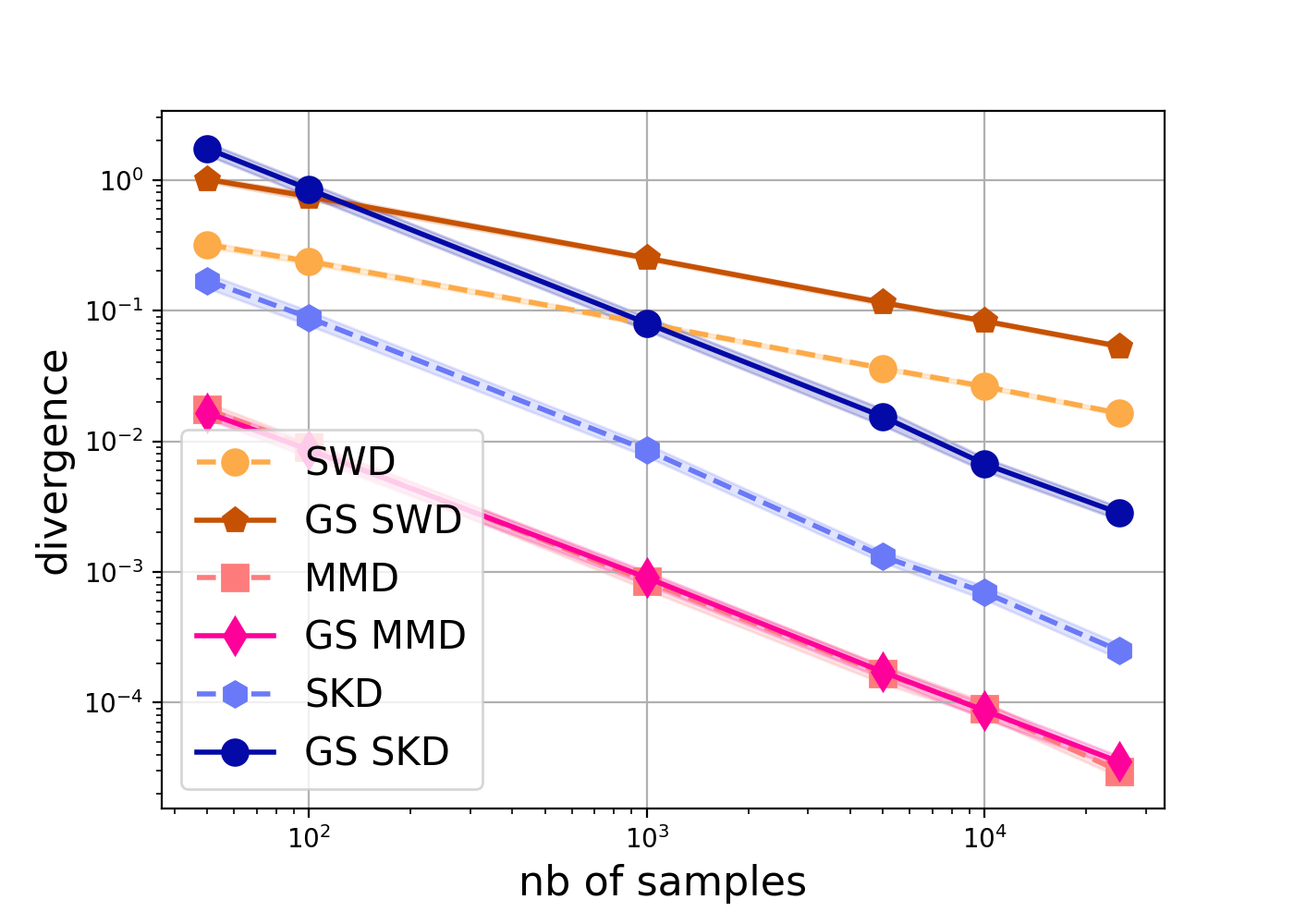} 
	\includegraphics[width=7cm]{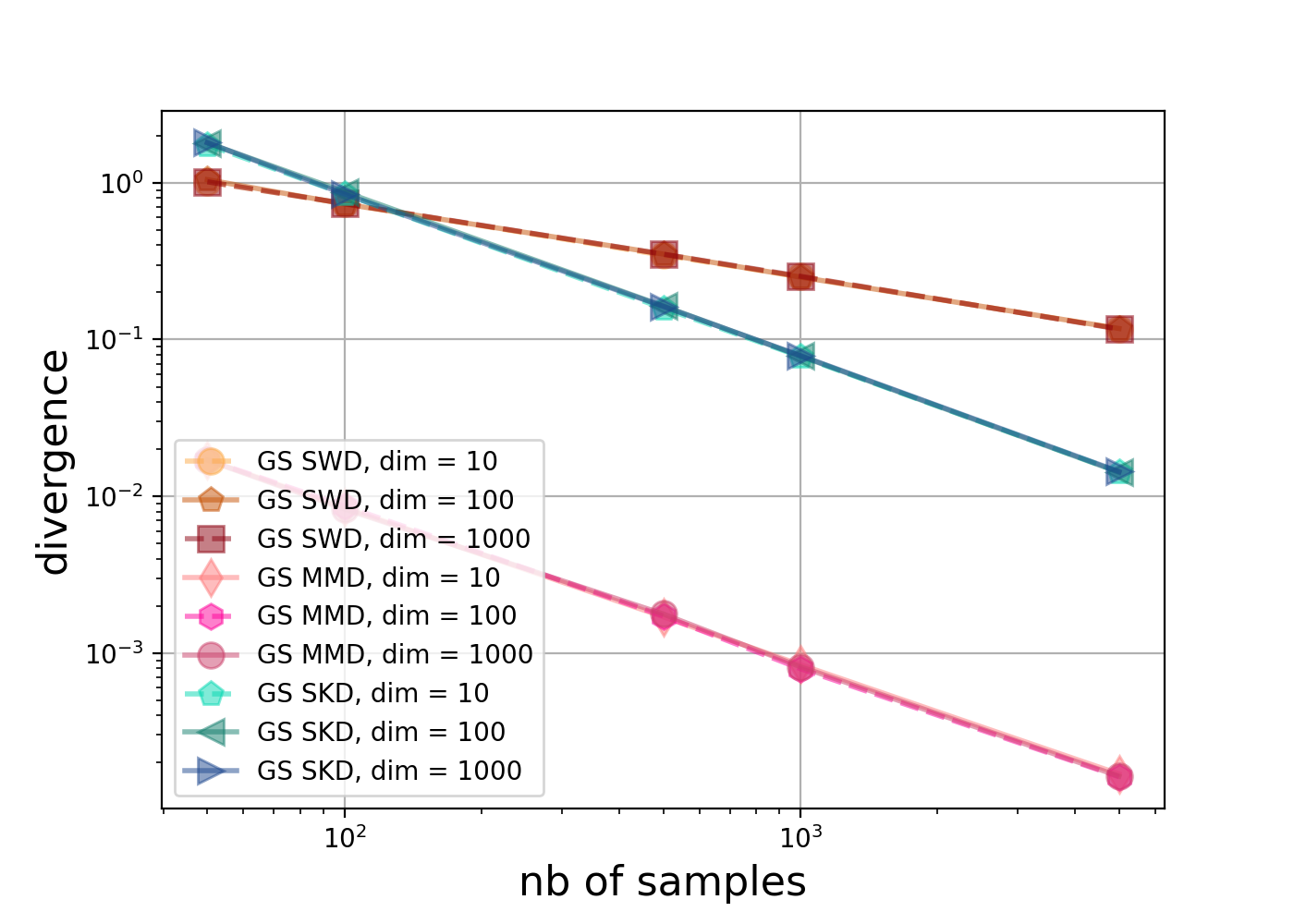} 
	
	\caption{Measuring the divergence between two sets of samples in $\R^{50}$, of increasing size,
		randomly drawn from $\mathcal{N}(0,\mathbf{I})$.
		We compare three sliced divergences and their Gaussian smoothed versions
		with a $\sigma=3$. (left) dimension has been set to $d=50$. (right) sample complexity       with different dimensions. This plot confirms that the 
		complexity is dimension-independent.
		\label{fig:divsamples}}  
\end{figure*}

\subsubsection{Projection complexity}

To compute the Gaussian smoothed sliced divergence, one may resort to  a Monte Carlo scheme to numerically approximate the integral in $G_\sigma SD^p(\mu,\nu)$. Towards this, let define the following sum:
\begin{align*}
\widehat{G_\sigma SD}^p(\mu,\nu) = \frac{1}{L}
\sum_{l=1}^L D^p(\mathcal{R}_{\u_l}\hat\mu_n * \mathcal{N}_\sigma,\mathcal{R}_{\u_l} \hat\nu_n* \mathcal{N}_\sigma),
\end{align*}
where $\u_l$ is a random vector uniformly drawn  from $\mathbb{S}^{d-1}$, for $l=1, \ldots, L.$ Theorem~\ref{theorem:error_MC} shows that for a fixed dimension $d$, the root mean square error of Monte Carlo approximation is of order $O\left(\frac{1}{\sqrt{L}}\right)$, which corresponds to the projection complexity.
\begin{theorem}
\label{theorem:error_MC}
Let $\mu, \nu \in \mathcal{P}(\R^d)$ and fix $p\in[1, \infty)$. Then, the error related to the Monte Carlo estimation of $G_\sigma SD^p$ is bounded as follows
\begin{align*}
\E[|\widehat{G_\sigma SD}^p( \mu,\nu) - G_\sigma SD^p(\mu,\nu)|] \leq \frac{A(p,\sigma)}{\sqrt{L}},
\end{align*}
where $A^2(p,\sigma) = \int_{\mathbb{S}^{d-1}}\big(D^p(\mathcal{R}_{\u}\mu * \mathcal{N}_\sigma,\mathcal{R}_{\u} \nu* \mathcal{N}_\sigma) - \bar\vartheta_p\big)^2\text{d}u_d(\u),$
with $\bar\vartheta_p = \int_{\mathbb{S}^{d-1}}D^p(\mathcal{R}_{\u}\mu * \mathcal{N}_\sigma,\mathcal{R}_{\u} \nu* \mathcal{N}_\sigma)\text{d}u_d(\u)$.
\end{theorem}
The term $A^2(p,\sigma)$ corresponds to the variance of $D^p(\mathcal{R}_{\u}\mu * \mathcal{N}_\sigma,\mathcal{R}_{\u} \nu* \mathcal{N}_\sigma)$ with respect to $\u \sim u_d$ drawn according to the uniform distribution over the unit-sphere $\mathbb{S}^{d-1}$.  It is worth to note that the precision of the Monte Carlo scheme approximation depends on the number of projections $L$ and the variance of the evaluations of the divergence $D^p.$ The estimation error decreases at the rate $L^{-1/2}$ according to  the number of projections used to compute the smoothed sliced divergence.

\begin{remark}Given the above results, we can provide a finer analysis of the Gaussian smoothed SWD sample complexity. 
For any $\mu, \nu \in \mathcal{P}_q(\R^d)$, the overall complexity of the Gaussian smoothed sliced Wasserstein distance is bounded by the sample and projection complexities, that is, 
\begin{equation*}
\complex(G_\sigma SWD^p(\mu,\nu)) = O\Big(\alpha_n(p,q,\sigma) + \frac{A(p,\sigma)}{\sqrt{L}}\Big).
\end{equation*}
If we consider the number of projections as $L = \lfloor n^{\beta}\rfloor $ for some $\beta \in (0, 1)$ then the overall complexity $\complex(G_\sigma SD^p(\mu,\nu)) = O(n^{-\beta/2})$. We further mention that complexity is  “interestingly” independent of the dimension $d.$

\end{remark}

\subsection{Noise-level dependencies}

When considered in sensitive applications requiring privacy preserving, the parameter $\sigma$ of the Gaussian smoothing function $\mathcal{N}_{\sigma}$ may significantly influence the attained privacy level. Hence, we provide theoretical results analyzing the effect of the noise level $\sigma$ on the induced Gaussian smoothed sliced divergence.

\paragraph{Order relation.} We first show that the noise level tends to reduce 
the difference between two distributions as measured using 
$G_\sigma S\text{D}^p(\mu,\nu)$ provided the base divergence $D$
satisfies some mild assumptions.

\begin{proposition}
\label{proposition:2-level-noise}
 Let $\mu$ and $\nu$ two distributions in $\cP(\R^d)$ and consider the noise levels
    $\sigma_1, \sigma_2$ such that 
    $0 \leq \sigma_1 \leq \sigma_2 < \infty$. Assume that the base divergence $D$ satisfies
    \begin{equation*}
      D^p(\mu'* \mathcal{N}_{\sigma_2}, \nu'* \mathcal{N}_{\sigma_2}) \leq 
   D^p(\mu' * \mathcal{N}_{\sigma_1}, \nu' * \mathcal{N}_{\sigma_1}), 
    \end{equation*}
for any $\mu', \nu' \in \mathcal{P}(\R).$ Then,
\begin{equation*}
G_{\sigma_2}S\text{D}^p(\mu,\nu) \leq G_{\sigma_1}S\text{D}^p(\mu,\nu).
  \end{equation*}  
\end{proposition}
\begin{proof}
   For all $\u\in\mathbb{S}^{d-1}$ we have $\mathcal{R}_\u \mu, \mathcal{R}_\u \nu \in \mathcal{P}(\R)$. By application of the inequality of noise level satisfied by $D^p_p$ in one dimension we get 
\begin{equation*}
D^p( \mathcal{R}_\u \mu * \mathcal{N}_{\sigma_2}, \mathcal{R}_\u \nu * \mathcal{N}_{\sigma_2}) \leq 
D^p(\mathcal{R}_\u \mu * \mathcal{N}_{\sigma_1}, \mathcal{R}_\u \nu * \mathcal{N}_{\sigma_1}).   
\end{equation*}
Then, computing the expectation over the projections $\u$ since the divergence is non-negative 
concludes the proof. \end{proof}
Note that the assumption for the base divergence inequality holds for the Gaussian smoothed Wasserstein distance~\cite{nietert2021smooth}. While we conjecture that it holds also
for smoothed Sinkhorn and MMD, we leave the proofs for future works.

Based on the property in Proposition~\ref{proposition:2-level-noise}, we can  show some specific properties of the metric 
with respect to the noise level $\sigma$. 
\begin{proposition}
 $G_{\sigma}S\text{D}^p(\mu,\nu)$ is decreasing with respect to 
    $\sigma$ and we have
\begin{equation*}
\lim_{\sigma \rightarrow 0}  G_{\sigma}S\text{D}^p(\mu,\nu) = \text{D}^p(\mu,\nu).
\end{equation*}
\end{proposition}
\begin{proof}
    The proof comes straightforwardly from Proposition~\ref{proposition:2-level-noise} by taking $\sigma_2 = 0$   and letting
    $\sigma_1 \rightarrow 0$. 
\end{proof}
This property interestingly states that the $G_{\sigma}S\text{D}^p$ recovers the sliced divergence when the noise level vanishes. 
{We end up this section by providing a relation between Gaussian smoothed sliced Wasserstein distances under two noise levels. Proof of Proposition \ref{proposition:GS-SWD_sigma_1_2} is postponed to the appendix.}
\begin{proposition}\label{proposition:GS-SWD_sigma_1_2}
Let $0\leq \sigma_1\leq \sigma_2$ be two noise levels. Then, one has
\begin{align*}
G_{\sigma_1} SWD^p(\mu,\nu) &\leq  2^{p-1} G_{\sigma_2} SWD^p(\mu,\nu)\\
&\qquad+ \frac{2\pi^{d/2}}{\Gamma(d/2)}2^{3p/2} (\sigma_2^2 - \sigma_1^2)^{2p},
\end{align*}
where $\Gamma:\R \rightarrow \R$ is the Gamma function expressed as $\Gamma(v) = \int_0^\infty t^{v-1}e^{-t}dt$. 
\end{proposition}

The above proposition allows to control the variation of the $G_{\sigma} SWD$ divergence with respects to the amount of Gaussian smoothing. 

\begin{remark} All these properties hold for the population case. When considering empirical approximation
	of the true distribution, it may not hold due to the impact 
	of $\mathcal{N}_\sigma$ over the sample complexity. 
\end{remark}

\begin{figure}
	\centering
        \includegraphics[width=7cm]{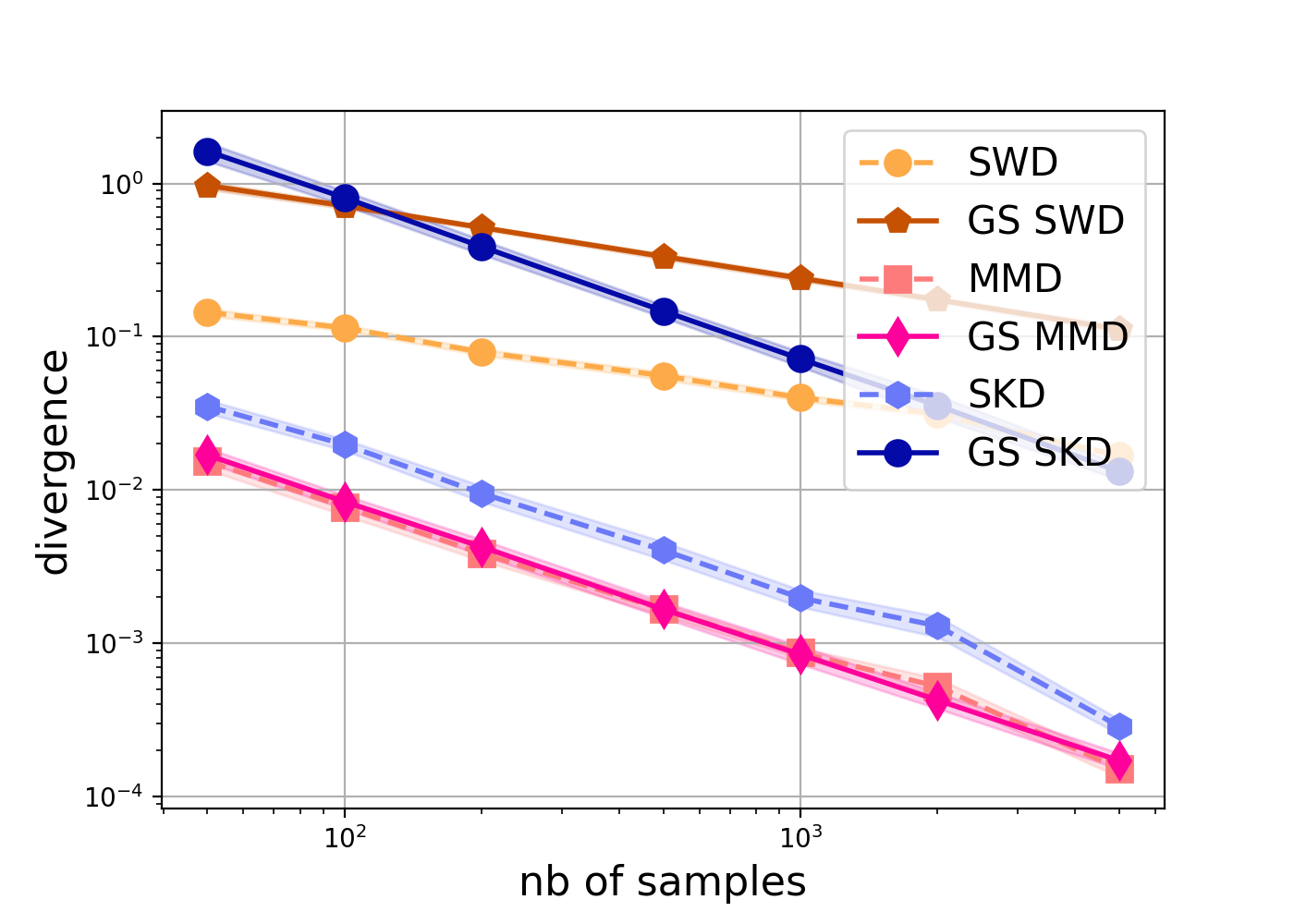} 
    \caption{Measuring the divergence between two sets of samples drawn iid from the CIFAR10
    dataset.
     We compare three sliced divergences and their Gaussian smoothed versions
     with a $\sigma=3$.\label{fig:cifar}}
\end{figure}

\begin{figure}[t]
	\centering
	\includegraphics[width=7cm]{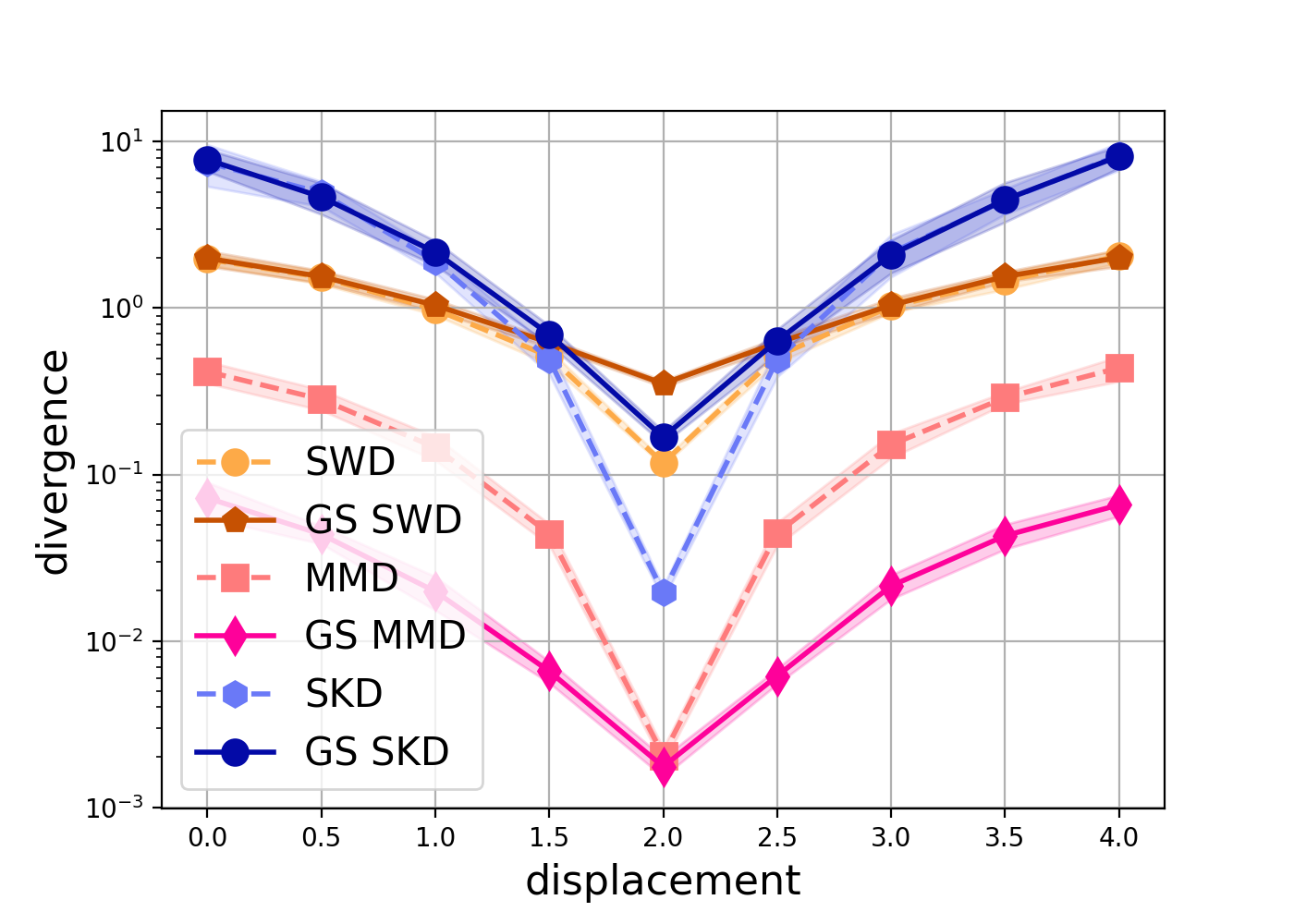}   
	\caption{Measuring the divergence between two sets of samples in $\R^{50}$, one
		with mean {$2  \mathbf{1}_d$} and the other with mean {$s \mathbf{1}_d$} 		with increasing $s$.
		We compare three sliced divergences and their Gaussian smoothed version
		with a $\sigma=3$.
		\label{fig:divdisplacement}}  
\end{figure}

\begin{figure}
	\centering
	\includegraphics[width=7cm]{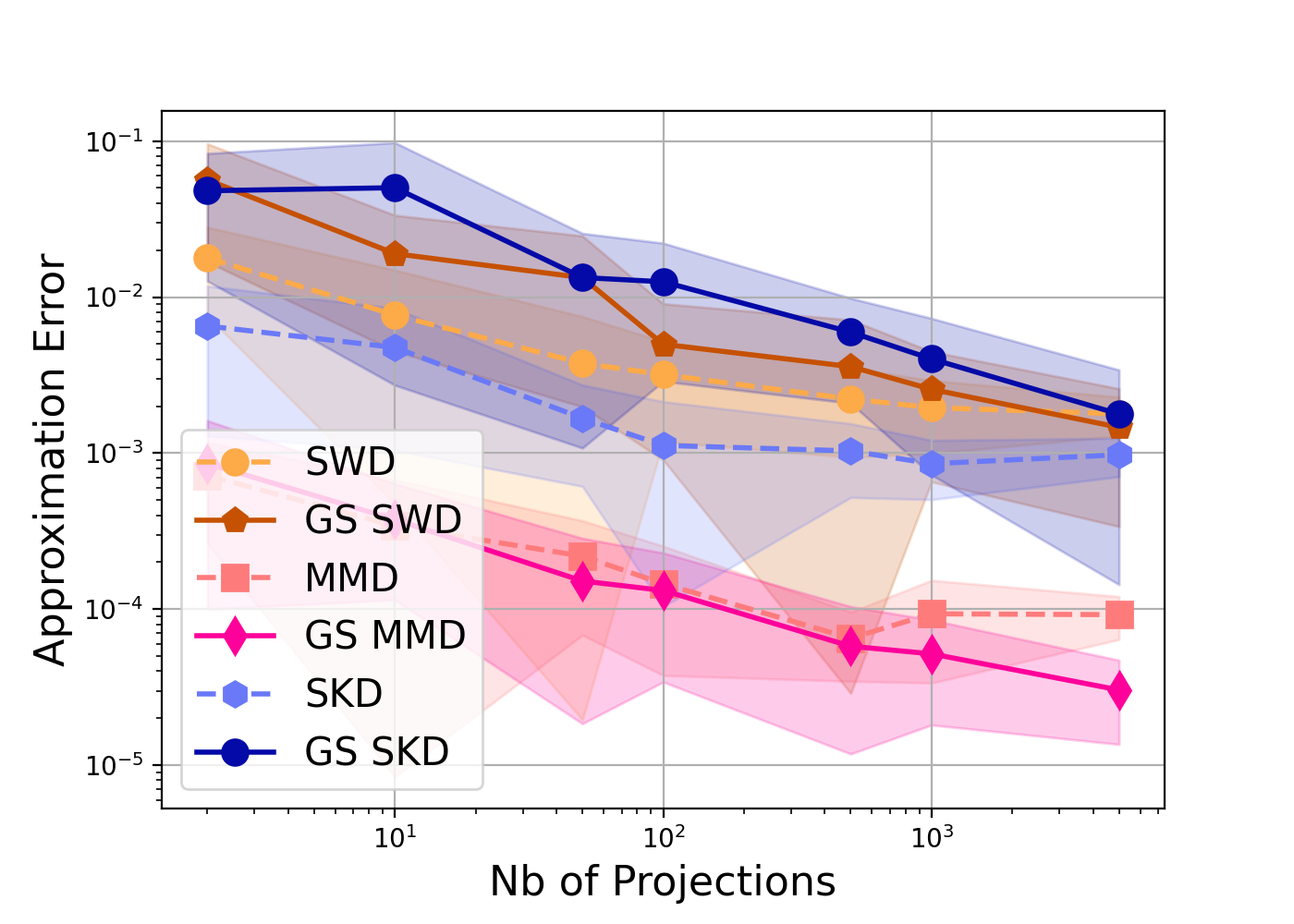}   
	\caption{Absolute difference between the approximated Monte-carlo approximation
		of all divergences compared to the true one (evaluated with $10,000$ number of projections).
		The  two sets of $500$ samples in $\R^{50}$ are  randomly drawn from $\mathcal{N}(0,\mathbf{I})$.
		The Gaussian smoothed divergences are parameterized with $\sigma=3$.
		\label{fig:approx}}  
\end{figure}
    \section{Numerical Experiments}  \label{sec:expe}
    \label{sec:expe}
    In this section, we report on a serie of experiments that support the theoretical results established in the previous section. We also highlight the usefulness of the findings     in a context of privacy-
    preserving domain adaptation problem.

\subsection{Supporting the theorical results}

\paragraph{Sample complexity}
The first  experiment (see Figure \ref{fig:divsamples}) analyzes the sample complexity of the different Gaussian smoothed 
sliced divergences.  It shows that the sample complexity stays similar to the one of their original 
and sliced counterparts up to a constant (see Theorem \ref{theorem:sample_complexity}). 
For this purpose, we have considered samples in $\mathbb{R}^d$ randomly drawn from a Normal distribution 
$\mathcal{N}(0,\mathbf{I})$.   For the Sinkhorn divergence, the entropy regularization has been set
to $0.1$ and for MMD, we used a Gaussian kernel for which the bandwidth has been set
to the mean of all pairwise distances between samples. 
The number of projections has been fixed
to $L=50$ and we perform 20 runs per experiment. For the first study, the convergence rate has been evaluated by increasing the samples number
up to 25,000 with fixed dimension $d=50$. For the second one, we vary both the dimension and the number of samples.

Figure \ref{fig:divsamples} shows the sample complexity of some sliced divergences, respectively noted as SWD, SKD and MMD for Sliced Wasserstein distance, Sinkhorn divergence and Maximum Mean discrepancy)
and their Gaussian-smoothed version, named as GS SWD, GS SKD and GS MMD. On the left plot, we can see that all 
Gaussian smoothed divergences preserve the complexity rate with just a slight to moderate overhead. The worst
difference is for Sinkhorn divergence, while smoothed MMD almost comes for free
in term of complexity. From the right plot where sample complexities for different dimensions $d$
are given, we confirm the finding that  Gaussian smoothing keeps the independence
of the convergence rate to the dimension of sliced divergences. 
We have also evaluated the sample complexity for the CIFAR dataset by sampling sets of increasing size.
Results reported in Figure \ref{fig:cifar} confirms the findings obtained from the toy dataset.

\paragraph{Identity of indescernibles}
The second experiment aims at checking whether our divergences converge towards a small value
when the distributions to be compared are the same. For this, { we consider samples from distributions $\mu$ and $\nu$  chosen as normal distributions with respectively}   mean $2 \times \mathbf{1}_d$ and $s \mathbf{1}_d$ with
varying $s$ (noted as the displacement). Results are depicted in Figure \ref{fig:divdisplacement}. We can see
that all methods are able to attain their minimum when $s=2$. Interestingly, the gap between
the Gaussian smoothed and non-smoothed divergences for Wasserstein and Sinkhorn is almost
indiscernible as the distance between distribution increases.

\paragraph{Projection complexity}
We have also investigated the impact of the number of projections when estimating the distance
between two sets of $500$ samples drawn from the same distribution, $\mathcal{N}(0,\mathbf{I})$.
Figure  \ref{fig:approx} plots the approximation error  between the true expectation of the
sliced divergences (computed for a number of $L=10,000$ projections) and its approximated
versions. We remark that, for all methods, the error ranges within $10$-fold when approximating
with $50$ projections and decreases with the number of projections. 

\paragraph{Impact of the noise parameter.} Since the Gaussian smoothing parameter
is  key in a privacy preserving context, as it impacts on the level of privacy of the Gaussian
mechanism, we have analyzed its impact of the  smoothed sliced divergence.
We have reproduced the experiment for the sample complexity but with different values
of $\sigma$. The number of projections has been set to $50$. 
Figure \ref{fig:samplenoise} shows these sample complexities. The first very interesting
point to note is that the smoothing parameter has almost no effect on the MMD sample
complexity. For the Gaussian smoothed SWD and Sinkhorn divergences, instead, the smoothing
tends to increase the divergence at fixed number of samples. Another interpretation
is that to achieve a given value of divergence, one needs more far samples when the
smoothing is larger (\emph{i.e} for getting a given divergence value at $\sigma=5$, one needs almost
$10$-fold more samples for $\sigma=15$). 
This overhead of samples needed when smoothing increases is properly
described, for the Gaussian smoothed SWD in our Proposition \ref{corollary:sampComplexGSWD}, as the sample complexity depends
on the moments of the Gaussian.

\begin{figure}
	\centering
    \includegraphics[width=7cm]{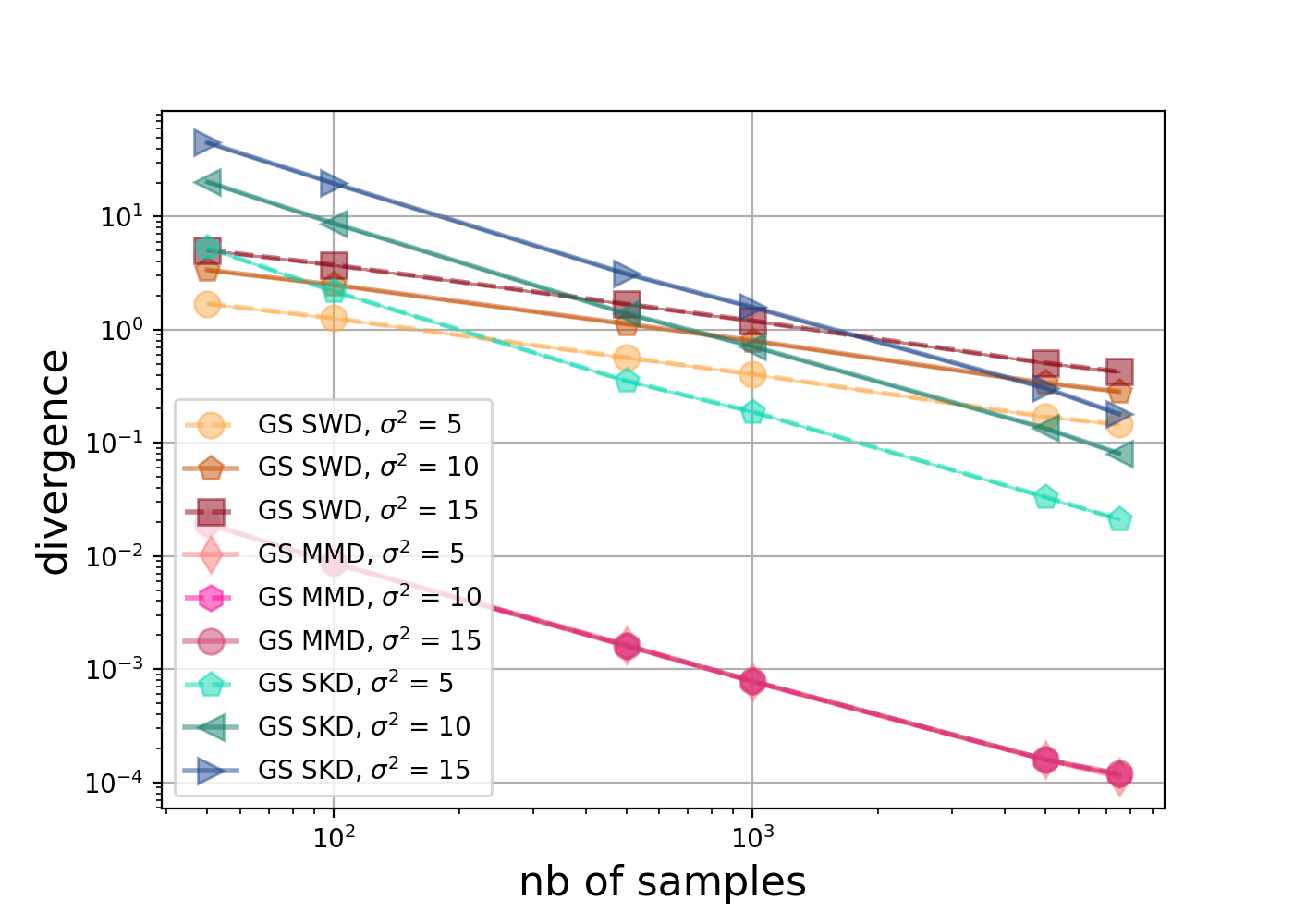}   
    \caption{Measuring the divergence between two sets of samples in $\R^{50}$ drawn
    from   $\mathcal{N}(0,\mathbf{I})$. We plot the sample complexity for different Gaussian smoothed 
    divergence at different level of noises.
    \label{fig:samplenoise}}
\end{figure}

As for conclusion from these analyses, we highlight that the Gaussian smoothed Sliced MMD seems
to present several strong benefits : its sample complexity does not depend on 
the dimension and seems to be the best one among the
divergence we considered. More interestingly, it is not impacted by the amount
of Gaussian smoothing and thus not impacted by a desired privacy level.

\subsection{Domain adaptation with Gaussian Smoothed Sliced Divergence}
\begin{figure*}[t]
	\begin{center}
		\includegraphics[width=5.6cm]{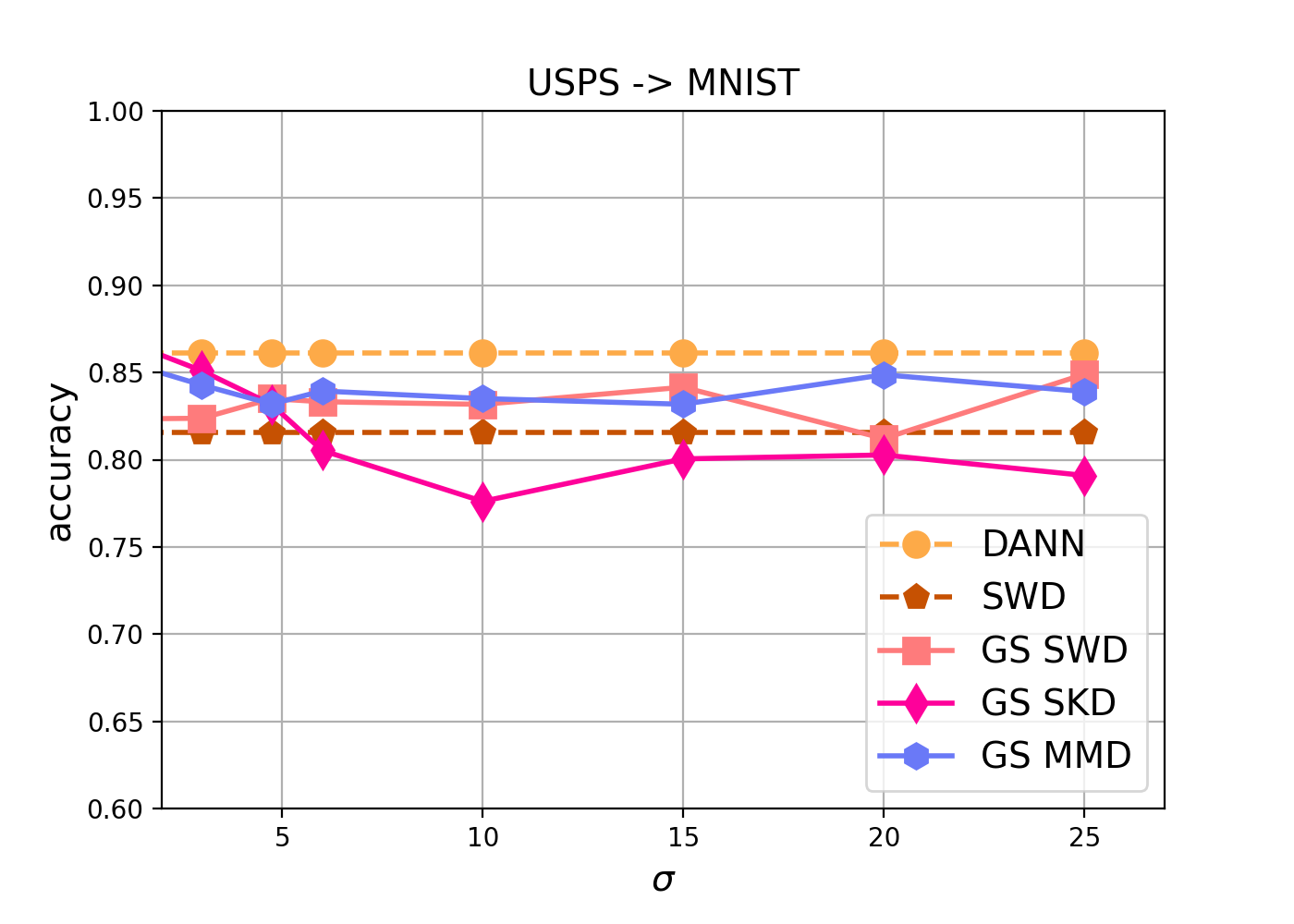}
		\includegraphics[width=5.6cm]{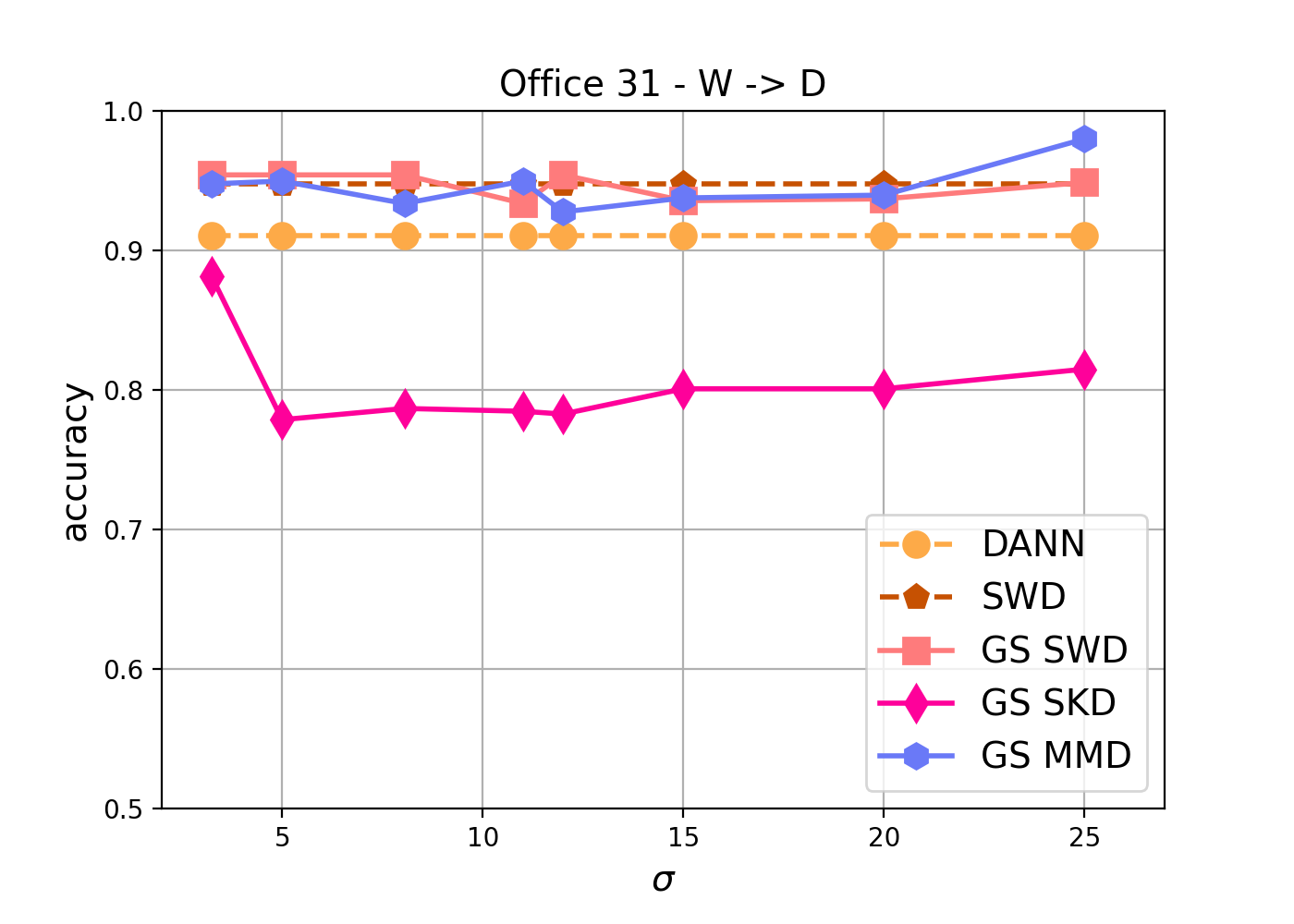}
		\includegraphics[width=5.6cm]{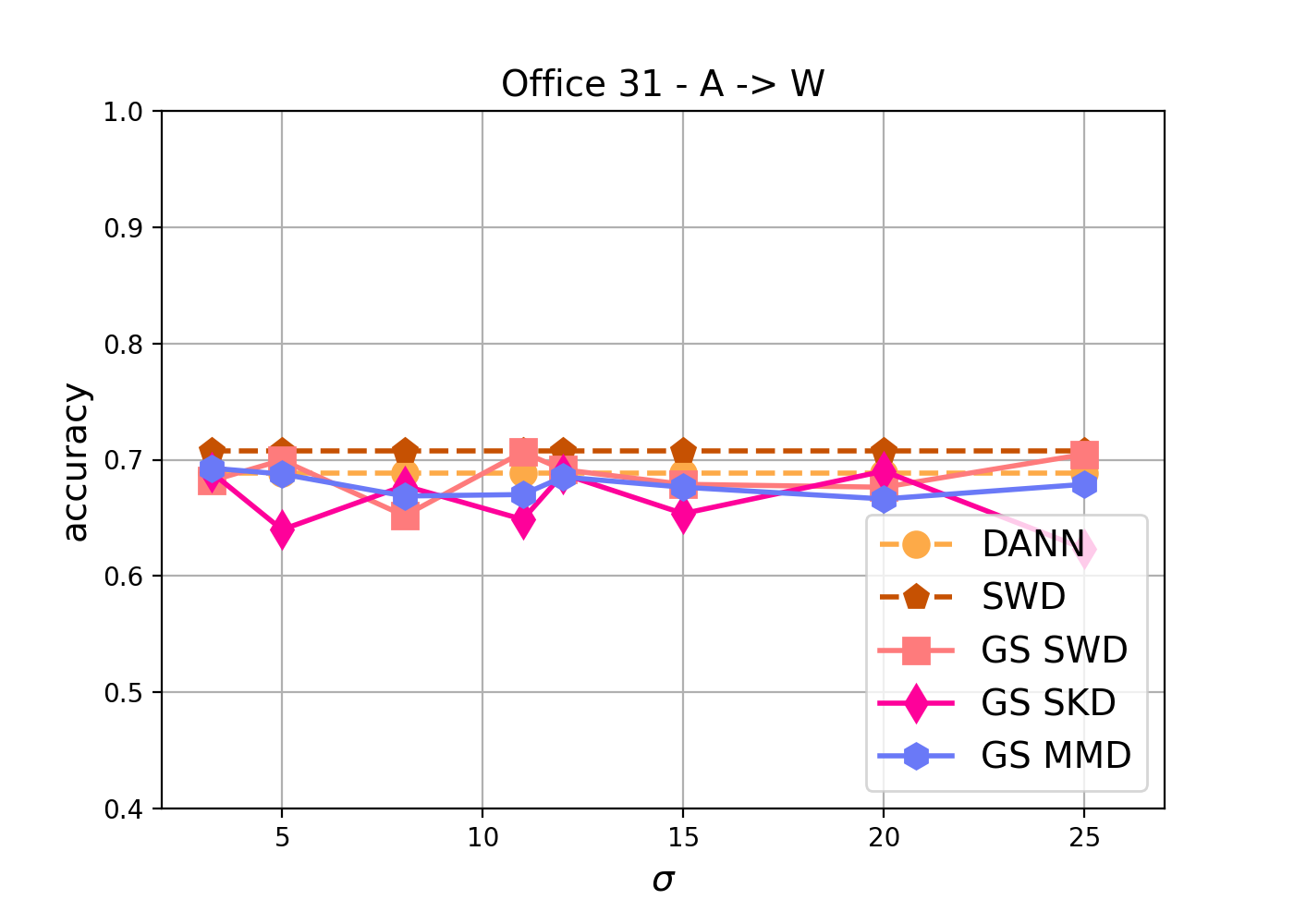}
	\end{center}
	\caption{Domain adaptation performances using different divergences on distributions with respect to the Gaussian smoothing. (left) USPS to MNIST. (middle) Office-31 Webcam to DSLR. (right) Office-31 Amazon to Webcam. \label{fig:da} }
\end{figure*}
As an application, we have considered the problem of unsupervised domain adaptation for a classification task. In this
context, given source examples $\X_s$ and their label $\y_s$ and unlabeled
	target examples $\X_t$, our goal is to design a classifier $h(\cdot)$ learned from 
	the source examples that generalizes well on the target ones. A classical
    approach consists in learning a representation mapping  $g(\cdot)$ that leads to
     invariant latent representations, invariance being measured as a distance
      between empirical distributions of mapped source and target samples. 
      Formally, this leads to the following  problem
	$$	
      \min_{g,h} L_c(h(g(\X_s)),\y_s) + \mathcal{D}(g(\X_s), g(\X_t))
	$$
      where $L_c$ can be the cross-entropy loss or a quadratic loss and 
    $\mathcal{D}$ a divergence between empirical distributions,
   in our case, $\mathcal{D}$ will be any Gaussian smoothed sliced divergence.
     We solve this problem through stochastic gradient descent, 
     similarly to many approaches that use Sliced Wasserstein Distance as a 
     distribution distance \cite{lee2019sliced}.
Note that, in practice,  using a smoothed divergence preserves the privacy of the 
target samples as shown by \citet{pmlr-v139-rakotomamonjy21a}.

Our experiments evaluate the studied  Gaussian-smoothed sliced divergences in classical unsupervised domain adaptation.
 We have considered twi datasets: 
a handwritten digit recognition (USPS/MNIST) and Office 31 datasets. Our goal is to analyze how our divergences 
 perform compared with non-smoothed divergences. The first one is the 
  Sliced Wasserstein Distance (SWD) \cite{lee2019sliced} and the second one is the Jenssen-Shannon approximation
  based on adversarial approach, known as DANN \cite{ganin2015unsupervised}. 
  For all methods and for each dataset, we used
the same neural network architecture for representation mapping and for
classification. Approaches differ only on how distance between distributions have been computed.

Results are depicted in Figure \ref{fig:da}. For the two problems, we can see that performances obtained with the Gaussian smoothed sliced Wasserstein or MMD divergences are similar to those obtained with DANN or SWD across all ranges of noise. 
The smoothed version of Sinkhorn is less stable and induces
a slight loss of performance. Owing to the metric property and
the induced weak topology, the privacy preservation comes almost
without loss of performance in this domain adaptation context.

\section{Conclusion}
\label{sec:conclu}
In this study, we have analyzed the properties of Gaussian smoothed sliced divergence for
comparing distributions as they play a crucial role in a privacy 
preserving context. We have derived several theoretical results related
to their topological and statistical properties. More precisely, 
we have shown that under mild condition on their base divergence, the smoothing
and slicing operations preserves metric property. From a statistical point of
view, we have shown that sample complexity does not depend on the dimension of
the problem and follows a similar complexity than their sliced version, although
some overhead may have to be paid due to the smoothing. 
We have illustrated those theoretical findings through 
some experimental analyses on toy problem. We have also
analyzed the behavior of our divergence on domain adaptation
problems and confirm the fact that using those divergences
yields only to slight loss of performances while preserving privacy.
One lesson we have also learnt is that Gaussian smoothed sliced MMD seems to present several strong benefits in term of sample complexity.

\providecommand{\CH}{{C.-H}}\providecommand{\JB}{{J.-B}}

 \bibliographystyle{icml2021} \clearpage

\appendix
\section{Additional definitions}

\subsection{Maximum Mean Discrepancy}
Let $k: \mathcal{X}\times \mathcal{X} \rightarrow \R$ be the reproducing kernel
of a reproducing kernel Hilbert space $\mathcal{H}$. The metric on
distance denoted as maximum mean discrepancy between $\mu$ and $\nu$
belonging to $\mathcal{P}(\mathcal{X})$ is defined as:
\begin{align*}
 MMD(\mu,\nu) = \left \| \int k(\cdot,x) d\mu(x) - \int k(\cdot,x) d\nu(x) \right \|_\mathcal{H}.   
\end{align*}
For empirical distributions, one can estimate the MMD using biased
or unbiased formulations as given by \citet{gretton2012kernel}:
For empirical distributions, one can estimate the MMD using biased
or unbiased formulations as given by \citet{gretton2012kernel}:
\begin{align*}
MMD(\hat \mu, \hat \nu) &=  \Bigg[\frac{1}{n^2}\sum_{i,j} k(x_i,x_j)
+ \frac{1}{m^2}\sum_{i,j} k(y_i,y_j)\\
& - \frac{2}{nm}\sum_{i,j} k(x_i,y_j) \Bigg]^{\frac{1}{2}}
\end{align*}

\subsection{Sinkhorn Divergence and Gaussian Smoothed Sliced Sinkhorn Divergence}
Let define the entropic regularized Wasserstein distance \cite{cuturi2013sinkhorn} between distributions $\mu$ and $\nu$ as 

\begin{align*}
	W_{p, \lambda}^p(\mu,\nu)
	&= \inf_{\gamma \in \Pi(\mu,\nu)} \int_{\R^d \times \R^d}
	\|x - y\|^p \gamma(x,y) dxdy \\
	& + \lambda H(\gamma|\mu \otimes \nu).
\end{align*}

where the set $\Pi(\mu,\nu)$ is defined as in Section \ref{sec:back}. The term $H(\cdot|\cdot)
$ is the relative entropy regularization of the transport plan with respect to the product measure $\mu \otimes \nu$, and  is given by 
$$H(\gamma|\mu \otimes \nu) = \iint \log\Big(\frac{\text{d}\gamma(x,y)}{\text{d}\mu \otimes \text{d}\nu(x,y)}\Big)\text{d}\gamma(x,y).$$ 
The related regularization parameter is $\lambda \geq 0$.  Then, the Sinkhorn divergence is defined as
$$
SKD_\lambda (\mu,\nu) = W_{p, \lambda}^p(\mu,\nu) - \frac{1}{2} W_{p, \lambda}^p(\mu,\mu) - \frac{1}{2} W_{p, \lambda}^p(\nu,\nu).
$$

Accordingly the Gaussian Smoothed Sliced Sinkhorn Divergence is expressed as
\begin{align*}
		G_\sigma& {SKD}_{p,\lambda}^p(\mu,\nu)\\
	&= \int_{\mathbb{S}^{d-1}}  SKD_{\lambda}^p(\mathcal{R}_\u \mu * \mathcal{N}_\sigma,\mathcal{R}_\u \nu* \mathcal{N}_\sigma) u_d(\u)d\u.
\end{align*}

\section{Proofs}

\subsection{Proof of Theorem~\ref{theorem:proof_topology}} \label{sub:proof_of_theorem_theorem:proof_topology}
$\bullet$ {\it Non-negativity (or symmetry).} The non-negativity (or symmetry) follows directly from the non-negativity (or symmetry) of $D^p$, see Definition 3.\\
$\bullet$ {\it Identity  property.} For the identity property, if the base divergence $D^p$ satisfies the identity property in one dimenstional measures, then for any $\mu \in \mathcal{P}(\R^d)$ and $\u \in \mathbb{S}^{d-1}$, one has that $D^p(\mathcal{R}_\u \mu * \mathcal{N}_\sigma,\mathcal{R}_\u \mu* \mathcal{N}_\sigma) =0,$ hence, by Definition 3, $G_\sigma SD^p(\mu,\mu) = 0.$
Let us now prove the fact that for any $\mu, \nu \in \mathcal{P}(\R^d), G_\sigma SD^p(\mu,\mu) =0$ entails $\mu = \nu$ a.s.
On one hand, $G_\sigma SD^p(\mu,\mu) = 0$ gives the fact that $D^p(\mathcal{R}_\u \mu * \mathcal{N}_\sigma,\mathcal{R}_\u \nu* \mathcal{N}_\sigma) = 0$ for $u_d$-almost every $\u \in \mathbb{S}^{d-1},$ hence $\mathcal{R}_\u \mu * \mathcal{N}_\sigma = \mathcal{R}_\u \nu* \mathcal{N}_\sigma$ for $u_d$-almost every $\u \in \mathbb{S}^{d-1}.$ Following the techniques in proof of Proposition 5.1.2 in~\cite{bonnotte:tel-00946781}, for any measure $\eta \in \mathcal{P}(\R^m)$ (with $m\geq 1$), $\mathcal{F}[\eta](\cdot)$ stands for the Fourier transform of $\s$ and is given as $\mathcal{F}[\eta](\v) = \int_{\R^m} e^{-i\s^\top \v}\text{d}\eta(\v)$ for any $\v \in \R^m.$ Then 
\begin{align*}
\mathcal{F}[\mathcal{R}_\u \mu * \mathcal{N}_\sigma](v) &= \int_\R e^{-i vt} \text{d}(\mathcal{R}_\u \mu * \mathcal{N}_\sigma)(t)\\
&=\int_{\R}\int_{\R}e^{-i(r+t)v}\text{d}\mathcal{R}_\u\mu(r) \text{d}\mathcal{N}_\sigma(t)\\
&= \int_{\R^d}\int_{\R} e^{-i(\langle \u, s\rangle+t)v}
\text{d}\mu(\s) \text{d}\mathcal{N}_\sigma(t)\\
&= \int_{\R} e^{-itv}\text{d}\mathcal{N}_\sigma(t) \int_{\R^d} e^{-i(\langle \u, s\rangle)v}
\text{d}\mu(\s)\\
&= \mathcal{F}[\mathcal{N}_\sigma](v)\mathcal{F}[\mu](v\u).
\end{align*}
Since for $u_d$-almost every $\u\in \mathbb{S}^{d-1}, \mathcal{R}_\u \mu * \mathcal{N}_\sigma = \mathcal{R}_\u \nu* \mathcal{N}_\sigma$, and hence $\mathcal{F}[\mathcal{R}_\u \mu * \mathcal{N}_\sigma] = \mathcal{F}[\mathcal{R}_\u \nu* \mathcal{N}_\sigma]
\Leftrightarrow \mathcal{F}[\mathcal{N}_\sigma]\mathcal{F}[\mu] = \mathcal{F}[\mathcal{N}_\sigma]\mathcal{F}[\nu]\Leftrightarrow \mathcal{F}[\mu] = \mathcal{F}[\nu]$. Since the Fourier transform is injective, we conclude that $\mu=\nu.$\\
$\bullet${\it Triangle inequality.} Assume that $D^p$ is a metric and let $\mu, \nu, \eta \in \mathcal{P}(\R^d).$ We then have  
\begin{align*}
&G_\sigma SD(\mu,\nu)\\
& = \Big\{\int_{\mathbb{S}^{d-1}}D^p(\mathcal{R}_\u\mu * \mathcal{N}_\sigma,\mathcal{R}_\u \nu* \mathcal{N}_\sigma) u_d(\u)\text{d} \u\Big\}^{1/p}\\
&\leq \Big\{\int_{\mathbb{S}^{d-1}}\Big(D(\mathcal{R}_\u\mu * \mathcal{N}_\sigma,\mathcal{R}_\u \eta* \mathcal{N}_\sigma)\\
&\qquad \qquad   + D(\mathcal{R}_\u\eta * \mathcal{N}_\sigma,\mathcal{R}_\u \nu* \mathcal{N}_\sigma)\Big)^pu_d(\u)\text{d} \u\Big\}^{1/p}\\
&\underbrace{\leq}_{(\star)} \Big\{\int_{\mathbb{S}^{d-1}}\Big(D^p(\mathcal{R}_\u\mu * \mathcal{N}_\sigma,\mathcal{R}_\u \eta* \mathcal{N}_\sigma)u_d(\u)\text{d} \u\Big\}^{1/p}\\
&\qquad \quad   + \Big\{\int_{\mathbb{S}^{d-1}} D^p(\mathcal{R}_\u\eta * \mathcal{N}_\sigma,\mathcal{R}_\u \nu* \mathcal{N}_\sigma)\Big)^pu_d(\u)\text{d} \u\Big\}^{1/p}\\
&= G_\sigma SD(\mu,\eta) + G_\sigma SD(\eta,\nu),
\end{align*}
where inequality in $(\star)$ follows from the application of Minkowski inequality.

\subsection{Proof of Proposition~\ref{corollary:sampComplexGSWD}} \label{sec:proof_of_corollary:sampcomplexgswd_}
Let us first upper bound the $k$-th moment of $M_k(\mathcal{R}_\u \mu * \mathcal{N}_\sigma)$, for all $k\geq 1.$
For all $\u\in \mathbb{S}^{d-1} $, one has 
\begin{align*}
M_k(\mathcal{R}_\u \mu * \mathcal{N}_\sigma)
&= \int_{\R}|t|^k \text{d}(\mathcal{R}_\u \mu * \mathcal{N}_\sigma)(t)\\
&= \int_{\R}\int_{\R}|r+t|^k \text{d}\mathcal{R}_\u\mu(r) \text{d}\mathcal{N}_\sigma(t)\\
&= \int_{\R^d}\int_{\R}|\langle \u, s\rangle+t|^k \text{d}\mu(\s) \text{d}\mathcal{N}_\sigma(t).
\end{align*}
Using the elementary inequality $(a+b)^k \leq 2^{k-1}(a^k + b^k)$ for $k\geq 1, a \geq 0,$ and $b\geq 0$, we obtain
\begin{align*}
&M_k(\mathcal{R}_\u \mu * \mathcal{N}_\sigma)\\
&\leq 2^{k-1}\int_{\R^d}\int_{\R}\big(|\langle \u, s\rangle|^k +|t|^k\big)\text{d}\mu(s) \text{d}\mathcal{N}_\sigma(t)\\
&\leq 2^{k-1}\Big(\norm{\u}\int_{\R^d}\norm{s}^k\text{d}\mu(s) + \int_{\R}|t|^k\text{d}\mathcal{N}_\sigma(t)\Big)\\
&\leq 2^{k-1}\Big(\int_{\R^d}\norm{s}^k\text{d}\mu(s) + \int_{\R}|t|^k\text{d}\mathcal{N}_\sigma(t)\Big)\\
&=2^{k-1}(M_k(\mu)) + M_k(\mathcal{N}_\sigma)).
\end{align*}
We then use the following result:
\begin{lemma}[see proof of Theorem 1 in~\cite{fournier2015}]
\label{lemma:thm1-fournier}
Let $\eta \in \mathcal{P}(\R)$ and let $p\geq 1$. Assume that $M_q(\eta)<\infty$ for some $q>p.$ There exists a constant $C_{p,q}$ depending only on $p,q$ such that, for all $n\geq 1,$
\begin{align*}
\E[W^p_p(\hat\eta_n, \eta)] \leq C_{p,q}M_q(\eta)^{p/q}\begin{cases}
n^{-1/2}{\bf{1}}_{q >2p }, \\
n^{-1/2}\log(n){\bf{1}}_{q=2p}\\
n^{-(q-p)/q}{\bf{1}}_{q \in (p ,2p)}.\end{cases}
\end{align*}
\end{lemma}
Let us fix $\mu \in \mathcal{P}_q(\R^d)$ with $q>p\geq1$ an empirical measure $\hat \mu_n$. Then, one has 
\begin{equation*}
M_q(\mathcal{R}_\u \mu * \mathcal{N}_\sigma) \leq 2^{q-1}(M_q(\mu)) + M_q(\mathcal{N}_\sigma)) < \infty.
\end{equation*}
By Lemma~\ref{lemma:thm1-fournier}, we obtain
\begin{align*}
\E&[G_\sigma SWD_p^p( \hat \mu_n,\mu)]\\
& = \E\bigg[\int_{\mathbb{S}^{d-1}}W_p^p(\mathcal{R}_\u \hat\mu_n * \mathcal{N}_\sigma,\mathcal{R}_\u \mu* \mathcal{N}_\sigma) u_d(\u)\text{d} \u \bigg]\\
&\leq C_{p,q}\begin{cases}
n^{-1/2}{\bf{1}}_{q >2p }, \\
n^{-1/2}\log(n){\bf{1}}_{q=2p}\\
n^{-(q-p)/q}{\bf{1}}_{q \in (p ,2p)}.\end{cases} \\
&\qquad  \times \int_{\mathbb{S}^{d-1}} M_q(\mathcal{R}_\u \mu * \mathcal{N}_\sigma)^{p/q} u_d(\u)\text{d} \u\\
&\leq C_{p,q}\begin{cases}
n^{-1/2}{\bf{1}}_{q >2p }, \\
n^{-1/2}\log(n){\bf{1}}_{q=2p}\\
n^{-(q-p)/q}{\bf{1}}_{q \in (p ,2p)}.\end{cases} \\
&  \times \begin{cases} 
\big(2^{q-1}(M_q(\mu)) + M_q(\mathcal{N}_\sigma))\big)^{p/q} {\bf 1}_{q \in 2\mathbb{N}^*},
 \\
\big(2^{q-1}(M_q(\mu))\big)^{p/q}{\bf 1}_{q \in 2\mathbb{N}+1}
.
\end{cases}
\end{align*}
On the other hand, since $W_p(\cdot, \cdot)$ is a metric, by applying Theorem~\ref{theorem:sample_complexity}, we obtain the following:
\begin{align*}
\E&[|G_\sigma SWD^p( \hat \mu_n,\hat \nu_n) - G_\sigma SWD^p(\mu,\nu)|]\\
&\leq 2 C_{p,q}\begin{cases}
n^{-1/2}{\bf{1}}_{q >2p }, \\
n^{-1/2}\log(n){\bf{1}}_{q=2p}\\
n^{-(q-p)/q}{\bf{1}}_{q \in (p ,2p)}.\end{cases} \\
& \times
\begin{cases}
\big(2^{q-1}(M_q(\mu, \nu)) + M_q(\mathcal{N}_\sigma))\big)^{p/q} {\bf 1}_{q \in 2\mathbb{N}^*}
 \\
\big(2^{q-1}(M_q(\mu, \nu))\big)^{p/q}{\bf 1}_{q \in 2\mathbb{N}+1}.
\end{cases}
\end{align*}

\subsection{Proof of Theorem~\ref{theorem:error_MC}} \label{sec:proof_of_theorem_theorem:error_mc}

Using Holder's inequality, we have 
\begin{align*}
&\E_{\u \sim u_d}\big[\big|\widehat{G_\sigma SD}^p(\mu,\nu) - {G_\sigma SD}^p(\mu,\nu)\big|\big]\\
&\leq \Big(\E_{\u \sim u_d}\big[\big|\widehat{G_\sigma SD}^p(\mu,\nu) - {G_\sigma SD}^p(\mu,\nu)\big|^2\big]\Big)^{1/2}\\
&= \Big(\V_{\u \sim u_d}\big[\big|\widehat{G_\sigma SD}^p(\mu,\nu)\big|\big]\Big)^{1/2}\\
&= \Big(\V_{\u \sim u_d}\big[\big|{G_\sigma SD}^p(\mu,\nu)\big|\big]\Big)^{1/2}\\
&= \frac{A(p,\sigma)}{\sqrt{L}}.
\end{align*}

\subsection{Proof of Proposition~\ref{proposition:GS-SWD_sigma_1_2}} \label{sub:proof_of_proposition_proposition:gs-swd_sigma_1_2}

The proof follows the same lines in proof of Lemma 1 in~\cite{pmlr-v139-nietert21a}.
First, we have that $\mathcal{N}_{\sigma_2} = \mathcal{N}_{\sigma_1} * \mathcal{N}_{\sqrt{\sigma_2^2 - \sigma_1^2}}$. Setting the following random variables: $X_\u\sim \mathcal{R}_\u \mu, Y_\u\sim\mathcal{R}_\u \nu, Z_X \sim \mathcal{N}_{\sigma_1}, Z_Y \sim \mathcal{N}_{\sigma_1}, Z'_X \sim \mathcal{N}_{\sqrt{\sigma_2^2 - \sigma_1^2}}, Z'_Y \sim \mathcal{N}_{\sqrt{\sigma_2^2 - \sigma_1^2}}$. The sliced Wasserstein distance $\text{W}_p^p(\mathcal{R}_\u \mu * \mathcal{N}_{\sigma_2},\mathcal{R}_\u \nu* \mathcal{N}_{\sigma_2})$ is given as a minimization over couplings $(X_\u, Z_X, Z'_X)$ and $(Y_\u, Z_Y, Z'_Y)$. Using the inequality $\E[|X|^p] - 2^{p-1}\E[|Y|^p]\leq 2^{p-1}\E[|X+Y|^p]$ for any random variables $X, Y \in \mathbb{L}_p$ integrable, we obtain, 
\begin{align*}
&2^{p-1}\E\big[|(X_\u + Z_X) - (Y_\u + Z_Y) + (Z'_X + Z'_Y)|^p\big]\\
&\geq \E\big[|(X_\u + Z_X) - (Y_\u + Z_Y) |^p\big]\\
&\qquad -2^{p-1}\E\big[|(Z'_X + Z'_Y)|^p\big]\big).
\end{align*}
Hence,
\begin{align*}
&2^{p-1}\text{W}^p_p(\mathcal{R}_\u \mu * \mathcal{N}_{\sigma_2},\mathcal{R}_\u \nu* \mathcal{N}_{\sigma_2})\\
&\geq \inf\Big(\E\big[|(X_\u + Z_X) - (Y_\u + Z_Y) |^p\big]\\
&\qquad -2^{p-1}\E\big[|(Z'_X + Z'_Y)|^p\big]\big)\Big)\\
&\geq \text{W}^p_p(\mathcal{R}_\u \mu * \mathcal{N}_{\sigma_1},\mathcal{R}_\u \nu* \mathcal{N}_{\sigma_1})\\
&\qquad - 2^{p-1}\sup\E\big[|(Z'_X + Z'_Y)|^p\big]\\
&\geq \text{W}^p_p(\mathcal{R}_\u \mu * \mathcal{N}_{\sigma_1},\mathcal{R}_\u \nu* \mathcal{N}_{\sigma_1}) - 2^p \sup\E\big[|(Z'_X)|^p\big].
\end{align*}
Therefore,
\begin{align*}
2^{p-1} G_{\sigma_2} SWD_p^p(\mu,\nu) &\geq G_{\sigma_1} SWD_p^p(\mu,\nu)|\\
&\qquad - 2^p u_d(\mathbb{S}^{d-1})\sup\E\big[|(Z'_X)|^p\big],
\end{align*}
hence,
\begin{align*}
G_{\sigma_1} SWD_p^p(\mu,\nu)| &\leq 2^{p-1} G_{\sigma_2} SWD_p^p(\mu,\nu)\\
& \qquad + 2^p u_d(\mathbb{S}^{d-1})\sup\E\big[|(Z'_X)|^p\big].
\end{align*}
Recall that if $Z\sim \mathcal{N}_{\sigma}$
\begin{align*}
\E[|Z|^p] = \frac{2^p\Gamma((p+1)/2)}{\Gamma(1/2)} \sigma^{2p} \leq 2^{p/2} \sigma^{2p}.	
\end{align*}
and $u_d(\mathbb{S}^{d-1}) = \frac{2\pi^{d/2}}{\Gamma(d/2)}$ then 
\begin{align*}
	G_{\sigma_1} SWD_p^p(\mu,\nu) &\leq 2^{p-1} G_{\sigma_2} SWD_p^p(\mu,\nu)+ \frac{2\pi^{d/2}}{\Gamma(d/2)}2^{3p/2} (\sigma_2^2 - \sigma_1^2)^{2p}.
\end{align*}

\end{document}